\documentclass[10pt,letterpaper]{article}

\usepackage[top=1in,bottom=1in,left=1in,right=1in]{geometry}

\usepackage[utf8]{inputenc} 
\usepackage[T1]{fontenc}    
\usepackage{hyperref}       
\usepackage{url}            
\usepackage{booktabs}       
\usepackage{amsfonts}       
\usepackage{nicefrac}       
\usepackage{microtype}      
\usepackage{xcolor}         

\usepackage{amsmath}
\usepackage{amssymb}
\usepackage{mathtools}
\usepackage{amsthm}
\usepackage{color}
\usepackage{algorithm}
\usepackage{algorithmic}

\theoremstyle{plain}
\newtheorem{theorem}{Theorem}[section]
\newtheorem{proposition}[theorem]{Proposition}

\newtheorem{assumption}[theorem]{Assumption}
\newtheorem{remark}[theorem]{Remark}
\newtheorem{claim}[theorem]{Claim}

\newcommand{\argmin}{\mathop{\mathrm{argmin}}}
\DeclareMathOperator*{\minimize}{minimize}
\newcommand{\subjectto}{{\rm subject\ to}}
\newcommand{\E}{\mathbb{E}}
\newcommand{\I}{\mathbf{I}}
\newcommand{\zero}{\mathbf{0}}
\newcommand{\tr}{\operatorname{tr}}
\renewcommand{\vec}{\mathbf{vec}}
\newcommand{\diag}{\mathbf{diag}}
\newcommand{\Diag}{\mathbf{Diag}}
\newcommand{\supp}{\mathrm{supp}}
\newcommand{\R}{\mathbb{R}}
\newcommand{\A}{\mathbf{A}}
\newcommand{\B}{\mathbf{B}}
\newcommand{\C}{\mathbf{C}}
\newcommand{\D}{\mathbf{D}}
\newcommand{\J}{\mathbf{J}}
\newcommand{\X}{\mathbf{X}}
\renewcommand{\S}{\mathbf{S}}
\renewcommand{\L}{\mathbf{L}}
\newcommand{\N}{\mathbf{N}}
\newcommand{\x}{\mathbf{x}}
\newcommand{\z}{\mathbf{z}}
\newcommand{\noise}{\boldsymbol{\varepsilon}}
\newcommand{\cC}{\mathcal{C}}
\newcommand{\F}{\mathcal{F}} 
\newcommand{\cL}{\mathcal{L}}
\newcommand{\M}{\mathcal{M}}
\newcommand{\cP}{\mathcal{P}}
\newcommand{\cR}{\mathcal{R}}
\newcommand{\T}{\mathcal{T}}
\newcommand{\cS}{\mathcal{S}}
\newcommand{\bOmega}{\mathbf{\Omega}}
\newcommand{\bSigma}{\mathbf{\Sigma}}
\newcommand{\bDelta}{\mathbf{\Delta}}
\newcommand{\full}{f}

\title{\textbf{Provable Guarantees and Efficient Learning of Structural Equation Models with Latent Confounders}}

\author{Weijian Yu\\
CIS, The University of Melbourne\\
\texttt{weijian.yu@student.unimelb.edu.au}\\
\and
Jean Honorio\\
CIS, The University of Melbourne\\
\texttt{jean.honorio@unimelb.edu.au}}

\date{}

\begin{document}

\maketitle

\begin{abstract}
Causal discovery aims to recover causal relationships from observed data.
In various fields, exploring causal relationships among variables remains an important topic, but this task becomes challenging due to the existence of latent confounders.
Ignoring such confounders can lead to false associations and incorrect edge directions.
In this paper, we study the linear structural equation model with latent confounders.
We propose an algorithm that iteratively identifies terminal (observed) nodes and reconstructs the directed acyclic graph of the observed variables.
To do this, we recover the precision matrix  of the observed variables as a sparse plus low-rank matrix: a sparse matrix captures the conditional dependencies among observed variables, while a low-rank matrix captures the combined influence of a few latent confounders.
We establish that for $p$ observed variables, $r$ latent confounders and $s$ edges, our procedure correctly identifies the directed causal relationship among observed variables, for $n \gtrsim \max\{s\log p,\ r p\}$ samples.
Experimental results validate our theoretical contributions.
\end{abstract}

\section{Introduction}

Causal discovery allows to infer causal relationships from observed data.
This is crucial for understanding complex systems such as genetics and finance, especially when control experiments are not feasible, costly or unethical.
A frequently adopted simplification is to assume causal sufficiency, i.e., there are no hidden confounding factors, and thus any association is fully determined by the other observed variables.
Real data rarely meets this requirement.
For instance, a psychological questionnaire reflects hidden psychological factors.
In vision and language, pixels and symbols are driven by underlying semantic factors.
These latent variables can induce interdependencies, masquerading as causal relationships.

Our goal is to discover directed causal relationships between observed variables, in the presence of latent confounders.
In this paper, we focus on the parameterization of the precision matrix (encoding undirected causal relationships) and the identification of directed causal relationships among observed variables.
The closest to our goal is fast causal inference~\cite{spirtes2000causation,colombo_learning_2012}, which does not impose any particular assumptions on latent variables, but outputs less information (i.e., a partial ancestral graph) and can potentially make an exponential number of conditional independence tests.
The case of undirected causal relationships between observed variables in the presence of latent confounders was previously studied by~\cite{chandrasekaran2010latent,meng2014learning}.
For directed causal relationships, while the fully observed case has been largely studied~\cite{ghoshal2018learning,zheng2018,ng2020} and we cannot do justice to the large body of work, there is a lack of methods that work in the presence of latent confounders.

In the literature, there are problems that include latent variables in directed causal discovery, but that do not relate to our goal.
\cite{Silva2003-SILLMM,huang2022latent,ng2024score} propose algorithms that can effectively learn directed causal relationships between latent variables and observed variables, as well as directed causal relationships among latent variables.
However, their methods do not focus on directed causal relationships between observed variables.
Another line work focuses on grouping observed variables that are likely affected by the same latent variable~\cite{silva06a,xie22a}.
Finally, a last line of work assumes that such latent-to-observed variable grouping is known~\cite{cui2018}.

\paragraph{Contributions.}

We develop a linear structural equation model (SEM) with latent confounders, which can effectively identify the causal relationships between observed variables.
In our model, the precision matrix follows a sparse plus low-rank decomposition.
We use a regularized maximum likelihood estimation (MLE) approach~\cite{chandrasekaran2010latent} to estimate the precision matrix (encoding undirected causal relationships).
Inspired by~\cite{yang2013dirty,meng2014learning}, under the conditions of Restricted Strong Convexity and Structural Incoherence, we derive error bounds for the regularized MLE.
Then we use the precision matrix to identify the directed causal relationships among the observed variables, using an algorithm similar to~\cite{ghoshal2018learning}.
We then provide provable recovery guarantees for our approach.
More specifically, we show that if the number of samples $n$ fulfills $n \gtrsim \frac{1}{\epsilon^2}\max\left\{s\log p,\ r p\right\}$ for $p$ observed variables, $r$ latent confounders and $s$ edges, then our method correctly recovers the directed causal edges.

\section{Preliminaries}

In this section, we first define some notations.
We then explain our main object of analysis: linear structural equation models (SEMs) with latent confounders.
We then present the connection between SEMs and undirected graphical models, which is important for our algorithm and guarantees.

\subsection{Notations}

We use bold lowercase to represent vectors, and bold uppercase to represent matrices.
We write $[p] := \{1,\ldots,p\}$.
We write the set $-i := [p]\setminus\{i\}$.
For a matrix $\A$, the trace is denotes as $\tr(\A)$ and $\vec(\A)$ denotes the vectorization of the matrix. 
The support of matrix $\A$ is denoted as $\supp(\A)=\{(i,j)\in[p]\times[p]\mid \A_{i,j}\neq0\}$.
$\diag(\A)$ extracts the main diagonal of $\A$ and $\Diag(\mathbf v)$ is the diagonal matrix with $\mathbf v$ on its diagonal.
We first introduce some notation.
For index sets $I,J\subseteq[p]$, $\A_{I,J}\in\R^{|I|\times|J|}$ is the sub-matrix of $\A\in\R^{p\times p}$ with rows in $I$ and columns in $J$; in this context, the symbol $\bullet$ denotes all rows (or columns).
For a matrix $\A$, $\|\A\|_{1}$ and $\|\A\|_{\infty}$ denote its entrywise $\ell_{1}$ and $\ell_{\infty}$ norms, respectively.
$\|\A\|_{2}$, $\|\A\|_{F}$, $\|\A\|_{*}$ denote the spectral, Frobenius and nuclear of matrix $\A$, respectively.

A directed graph of $p$ nodes is denoted as $G=([p],E)$ with edges $E\subseteq[p]\times[p]$, where $(i,j)\in E$ represents a directed edge from $j$ to $i$.
A directed acyclic graph (DAG) is a directed graph without cycles (i.e., starting from any node and following the direction of the edges, one cannot return to the original node).
For node $i$, $\pi_G(i)$ denote the sets of parents and $\phi_G(i)$ denote the sets of children in graph $G$.
We call node $i$ terminal if $\phi_G(i)=\varnothing$.
Let $\T_G$ be a set of topological orderings over $[p]$ in $G$.
$\T_G=\{\tau\in \cS_p | \tau(j)<\tau(i)\ \text{if}\ (i,j)\in E\}$, where $\cS_p$ is the set of all possible permutations of $[p]$.
For any topological order $\tau\in\T_G$
and any $m\in[p]$, define the sequence of graphs $G[m,\tau] = \bigl(V[m,\tau],\,E[m,\tau]\bigr)$
where $G[m,\tau]$ is the induced subgraph of $G$ on the first $m$ nodes in
the ordering $\tau$, i.e., $V[m,\tau] =  \{\, i\in[p]\;|\; \tau(i)\le m \,\}$ and 
  $E[m,\tau]=\{\, (i,j)\in E \;|\; i\in V[m,\tau],\, j\in V[m,\tau] \,\}$.
Equivalently, $G[m,\tau]$ contains exactly the first $m$ nodes under $\tau$
and all the edges among them.

\subsection{A Linear SEM with Latent Confounders}

Let $p$ be the number of observed variables and $r<p$ be the number of latent confounders.
Let $\B \in \R^{p \times p}$ encode the causal effects between observed variables, $\C \in \R^{p \times r}$ encode the causal effects from latent confounders to observed variables, and $\D \in \R^{r \times r}$ encode the causal effects among latent confounders.
In other words, the edges in the DAG is defined by the support of the matrices $\B$, $\C$ and $\D$.

In a linear SEM, the random observed vector $\x \in \R^p$ and the latent vector $\z \in \R^r$ can be written as the linear combination
$\x = \B \, \x + \C \, \z + \noise_X$ and
$\z = \D \, \z + \noise_Z$,
or equivalently,
\begin{equation} \label{eq:sem}
\begin{bmatrix} \x \\ \z \end{bmatrix}
=
\begin{bmatrix}
\B & \C \\
\zero & \D
\end{bmatrix}
\begin{bmatrix} \x \\ \z \end{bmatrix}
+
\begin{bmatrix} \noise_X \\ \noise_Z \end{bmatrix},
\end{equation}
where the noise variables fulfill $\E[\noise_X] = \zero$, $\E[\noise_Z] = \zero$ and $\E\left[\begin{bmatrix} \noise_X \\ \noise_Z \end{bmatrix} \begin{bmatrix} \noise_X \\ \noise_Z \end{bmatrix}^\top\right] = \Diag\!\left(\{\sigma_i^2\}_{i=1}^{p+r}\right)$.
We denote the SEM over the observed variables as $(G,\B,\{\sigma_i^2\}_{i=1}^p)$ where $G=([p],E)$ and $E=\supp(\B)$.
We denote the SEM over observed variables and latent confounders as $\left(G_{\full},\begin{bmatrix}
\B & \C \\
\zero & \D
\end{bmatrix},\{\sigma_i^2\}_{i=1}^{p+r}\right)$ where $G_\full=([p+r],E_\full)$ and $E_\full=\supp\left(\begin{bmatrix}
\B & \C \\
\zero & \D
\end{bmatrix}\right)$.

We assume we have $n$ samples from the true distribution, but which access to only $p$ observed variables ($\x$).
That is, we do not have access to $r$ latent confounders ($\z$).
Our goal is to recover $\B$ (and its support) by using only observed data.
Formally speaking, we receive a data matrix \(\X\in \R^{n\times p}\) of observed variables, which comes from \(\left(G_{\full}, \begin{bmatrix}
\B^* & \C^* \\
\zero & \D^*
\end{bmatrix}, \{\sigma_i^2\}_{i=1}^{p+r}\right)\),
and we want to recover a SEM \((\widehat{G}, \widehat{\B}, \{\widehat{\sigma}_i^{2}\}_{i=1}^p)\) such that \(G^*=\widehat{G}\) or equivalently, such that $\supp(\widehat{\B})=\supp(\B^*)$.

\subsection{SEMs and Undirected Graphical Models with Latent Confounders}

While linear SEMs are directed graphical models, we can also view them as undirected graphical models.
We leverage this connection later in Section~\ref{sec:algoandguarantees} to motivate a new algorithm for learning SEMs from data, with provable guarantees.
Define the (full) covariance matrix as follows:
\begin{equation} \label{eq:Sigmafull}
\bSigma_{\full} := \E\left[\begin{bmatrix} \x \\ \z \end{bmatrix} \begin{bmatrix} \x \\ \z \end{bmatrix}^\top\right] := \left[\begin{matrix} \bSigma & \bSigma_{X,Z} \\ \bSigma_{X,Z}^\top & \bSigma_{Z,Z}  \end{matrix}\right],
\end{equation}
where $\bSigma := \E\left[\x\x^\top\right] \in \R^{p \times p}$, $\bSigma_{X,Z} \in \R^{p \times r}$ and ${\bSigma_{Z,Z} \in \R^{r \times r}}$.
Let $\J_\full$ be the precision matrix related to the (full) covariance matrix $\bSigma_{\full}$, i.e.,
\begin{equation} \label{eq:Jfull}
\J_\full := \bSigma_{\full}^{-1} := \left[\begin{matrix} \S & \J_{X,Z} \\ \J_{X,Z}^\top & \J_{Z,Z}  \end{matrix}\right],
\end{equation}
where $\S \in \R^{p \times p}$, $\J_{X,Z} \in \R^{p \times r}$ and $\J_{Z,Z} \in \R^{r \times r}$.
Let $\bOmega\in\R^{p\times p}$ be the marginal precision matrix related to $\bSigma = \E\left[\x\x^\top\right]$, i.e., $\bOmega := \bSigma^{-1}$.
It can be shown that (see Appendix~\ref{app:covarianceprecision} for the derivations),
\begin{align} \label{eq:OmegaSL}
\bOmega 
= \S + \L, 
\qquad
\S 
= (\I-\B)^{\top} \N_X^{-1}(\I-\B), 
\qquad
\L 
= -\J_{X,Z}\,\J_{Z,Z}^{-1}\,\J_{X,Z}^\top .
\end{align}
where $\N_X = \Diag\!\left(\{\sigma_i^2\}_{i=1}^p\right)$.
First, note that $\S$ captures the condition dependencies among observed variables, while $\L$ captures the commbined influence of latent confounders.
For this reason, we call $\S$ the confounder-free precision matrix.
Moreover, $\S$ is sparse and $\L$ is a matrix of rank at most $r$ (since $\J_{Z,Z} \in \R^{r \times r}$).
Thus, the marginal precision matrix $\bOmega$ can be written as the sum of a sparse component and a low-rank component.

\begin{remark}
From eq.\eqref{eq:OmegaSL}, one can observe that successful recovery of $\S$ from training data, would imply successful recovery of $\B$.
This motivates our algorithm and the study of its theoretical guarantees.
\end{remark}

Given $n$ samples for the observed variables only, i.e., $\X \in \R^{n \times p}$, we define the observed sample covariance matrix as $\widehat{\bSigma}=\frac{1}{n}\, \X^\top\X$.
The regularized MLE for undirected graphical models with latent confounders~\cite{chandrasekaran2010latent,meng2014learning} solves the following optimization problem:
\begin{align} \label{eq:sparselowrankproblem}
 \minimize_{\S,\,\L} \quad & \cL(\S+\L;\,\X) \;+\; \lambda \|\S\|_{1} \;+\; \mu \|\L\|_{*} \nonumber \\
 \subjectto \quad & -\,\L \succeq 0,\quad \S + \L \succeq 0 ,
\end{align}
where $\lambda,\mu>0$ are regularization constants and
$\cL(\bOmega;\X)
:= \langle \widehat{\bSigma},\, \bOmega \rangle
- \log\det(\bOmega)$ is the negative log-likelihood function.
In eq.\eqref{eq:sparselowrankproblem} the $\ell_1$ norm regularizer on $\S$ encourages sparsity since the $\ell_1$ norm is a convex surrogate for the number of none-zero entries in $\S$.
Similarly, the nuclear norm regularizer on $\L$ encourages low-rankness since the nuclear norm is a convex surrogate for the rank of $\L$.

\section{Main Results}

In this section, we first discuss the theoretical framework needed for the analysis, which borrows from undirected graphical models.
We then provide recovery guarantees for the confounder-free precision matrix.
Then, we turn our attention to directed graphical models and present a sufficient and necessary condition for identifiability.
Armed with those results, we end the section by presenting our algorithms and its provable theoretical guarantees.

\subsection{Decomposable Regularization for Undirected Graphical Models with Latent Confounders}

Our results build on the framework of~\cite{negahban2012unified,yang2013dirty} for estimation with superposition of structurally constrained parameters.
This framework was also used in the estimation of sparse and low-rank undirected graphical models in~\cite{meng2014learning}.
Next, we discuss the different definitions and assumptions relevant this problem.

\paragraph{Decomposable regularizers.}

We first review some terms in~\cite{negahban2012unified}.
Let $\M$ be the model subspace which captures constraints on the model parameters and $\overline{\M}^{\perp}$ be the perturbation subspace with perturbations far from the model subspace.
$(\M,\overline{\M}^{\perp})$ denote a pair of subspaces, where $\M\subseteq\overline{\M}$.
A regularization function $\cR(\cdot)$ is called decomposable for a subspace pair $(\M,\overline{\M}^{\perp})$ if
$\cR(u+v)=\cR(u)+\cR(v),\;\;\text{for all}\;\; u\in \M \;\;\text{and}\;\; v\in \overline{\M}^{\perp}$.

The $\ell_{1}$ norm is decomposable for sparse matrices.
Let $E\subseteq [p] \times [p]$ be a set of index pairs where the entries of the sparse matrix is non-zero.
Let $\M(E) = \overline{\M}(E)$ denote the subspace of all sparse matrices in $\R^{p \times p}$ supported in the subset of $E$. Let $E^c$ be the complement of $E$.
Note that $\|\A\|_{1}=\|\A_{E}\|_{1}+\|\A_{E^c}\|_{1}$ which implies decomposability.

The nuclear norm is decomposable for symmetric positive semi-definite low-rank matrices, as shown in~\cite{meng2014learning}.

Let $(\M,\overline{\M}^{\perp})$ be a pair of subspaces. Following~\cite{negahban2012unified}, we define the
\emph{structural error set} at $\bOmega^*$ by
\begin{align*}
& \cC(\M,\overline{\M}^{\perp};\bOmega^*)
:= 
\Bigl\{\bDelta \in \R^{n\times p}\; \Big| \;
\cR(\bDelta_{\overline{\M}^{\perp}})
\;\le\; 3\,\cR(\bDelta_{\overline{\M}}) + 4\,\cR\bigl(\bOmega^*_{\overline{\M}^{\perp}}\bigr)
\Bigr\},
\end{align*}
In essence, if the true precision matrix $\bOmega^*$ has only a small component in the orthogonal complement $\overline{\M}^{\perp}$, then any $\bDelta$ from $\cC$ must have a small projection onto $\overline{\M}^{\perp}$.

Let the true precision matrix $\bOmega^*$ decompose into a sparse
component $\S^*$ and a low-rank component $\L^*$, i.e., $\bOmega^* = \S^* + \L^*$.
For the sparse component $\S^*$, we define a pair of subspaces $\bigl(\M(E),\overline{\M}^{\perp}(E)\bigr)$, and we let $\cC(E) \;=\;
\cC\bigl(\M(E),\M(E)^\perp;\; \S^\ast\bigr)$.
For the low-rank component $\L^*$, we define a pair of subspaces $\bigl(\M(U),\overline{\M}^{\perp}(U)\bigr)$, and we let $\cC(U) \;=\;
\cC\bigl(\M(U),\M(U)^\perp;\; \L^\ast\bigr)$.
In later analysis, perturbations of $\bOmega^*$ are restricted to the directions in these two sets.

\paragraph{Restricted Strong Convexity (RSC).}

Given some set $\cC$, the loss $\cL$ satisfies RSC~\cite{negahban2012unified} on $\cC$ if there exists a tolerance function $\tau_{\cL}$ and some curvature parameter
$\kappa_{\cL}>0$ such that
\begin{equation*}
\delta\cL(\bDelta;\bOmega^*)
\;\ge\;
\kappa_{\cL}\|\bDelta\|_{F}^{2} - \tau_{\cL}(\bOmega^*),
\qquad \forall\, \bDelta\in\cC.
\end{equation*}
where $\bDelta = \bOmega^* - \bOmega$ and $\delta\cL$ is the first-order Taylor remainder of the loss $\cL$ at $\bOmega^*$, i.e.,
$\delta\cL(\bDelta;\bOmega^*)
= \cL(\bOmega^*+\bDelta)-\cL(\bOmega^*)
- \big\langle \nabla \cL(\bOmega^*), \bDelta \big\rangle$.

\paragraph{Structural Incoherence (SI).}

To control the interaction between the sparse and the low-rank components, we assume that $\cL$ satisfies the SI condition~\cite{yang2013dirty}.
That is, for all $\bDelta_S\in\cC(E),\;
\bDelta_L\in\cC(U)$:
\begin{equation*}
c_{\cL}(\bDelta_S,\bDelta_L;\bOmega^*)
\;\le\;
\frac{\kappa_{\cL}}{2}\Big(\|\bDelta_S\|_{F}^{2}
+\|\bDelta_L\|_{F}^{2}\Big),
\end{equation*}
where $\kappa_{\cL}$ is defined as in RSC, and $c_{\cL}$ is the incoherence function, defined as
$c_{\cL}(\bDelta_S,\bDelta_L;\bOmega^*)
= |\cL(\bOmega^*+\bDelta_S+\bDelta_L)+\cL(\bOmega^*) 
  -\,\cL(\bOmega^*+\bDelta_S)
  -\,\cL(\bOmega^*+\bDelta_L)|$.

Recall that we use the (undirected graphical model) optimization problem in eq.\eqref{eq:sparselowrankproblem}.
Unfortunately, the analysis of~\cite{meng2014learning} does not provide a recovery guarantee for $\S$, but for $\bOmega=\S+\L$.
In this paper, we follow similar assumptions as in~\cite{meng2014learning}, but provide a recovery guarantee for $\S$.

The following two assumptions for the Fisher information from~\cite{meng2014learning} allows to show that the problem in eq.\eqref{eq:sparselowrankproblem} fulfills the RSC and SI conditions.
The Fisher information at the true precision matrix $\bOmega^*$ is
$\F^* = \bOmega^{*-1}\otimes \bOmega^{*-1}$,
where $\otimes$ denotes the Kronecker product. The \emph{Fisher inner product}
between matrices $\bDelta_A$ and $\bDelta_B$ is defined as
$\langle \bDelta_A,\bDelta_B\rangle_{\F^*}
  := \vec(\bDelta_A)^{\top}\F^*\vec(\bDelta_B)
  = \tr\!\big(\bOmega^{*-1}\bDelta_A\,\bOmega^{*-1}\bDelta_B\big)$.
This inner product induces the \emph{Fisher norm}~\cite{kakade2010learning}, formally defined as
$\|\bDelta\|_{\F^*}^2
  := \vec(\bDelta)^{\top}\F^*\vec(\bDelta) 
  = \tr\!\big(\bOmega^{*-1}\bDelta\,\bOmega^{*-1}\bDelta\big)$.

\begin{assumption}[{\normalfont \textbf{Restricted Fisher Eigenvalue}}, Assumption 1 in~\cite{meng2014learning}]
\label{assum:rfe}
There exists a constant $\kappa_{\min}^{\ast}>0$ such that
\begin{equation*}
  \|\bDelta\|_{\F^{\ast}}^{2}
  \;\ge\; \kappa_{\min}^{\ast}\,\|\bDelta\|_{F}^{2},
  \qquad \forall\,\bDelta\in \cC(E)\cup \cC(U).
\end{equation*}
\end{assumption}
This RFE condition generalizes the restricted eigenvalue condition for sparsity-promoting linear regression problems~\cite{bickel2009simultaneous}.

Let
$\cP_{E}=\cP_{\overline{\M}(E)}$,
$\cP_{U}=\cP_{\overline{\M}(U)}$,
$\cP_{E^\perp}=\cP_{\overline{\M}(E)^\perp}$,
$\cP_{U^\perp}=\cP_{\overline{\M}(U)^\perp}$
be the projection operator onto the subspaces $\overline{\M}(E)$,$\overline{\M}(U)$, $\overline{\M}(E)^\perp$, $\overline{\M}(U)^\perp$, respectively. We assume the following conditions for the Fisher information.

\begin{assumption}[\textbf{Structural Fisher Incoherence}, Assumption 2 in~\cite{meng2014learning}]
\label{assum:sfi}
Given $M>6$, and 
define the subspace pairs 
$\bigl(\M(E),\overline{\M}^{\perp}(E)\bigr)$ and $\bigl(\M(U),\overline{\M}^{\perp}(U)\bigr)$. Let   $\Lambda \;=\; 2 + 3\max\!\left\{
  \frac{\lambda\sqrt{s}}{\mu\sqrt{r}},\;
  \frac{\mu\sqrt{r}}{\lambda\sqrt{s}}
\right\}$, where $s = |E|$ is the number of elements in $E$ and $r = {\rm rank}(U)$.
Given regularization parameters $\lambda$ and $\mu$, then the Fisher information $\F^*$ satisfies:
\begin{align*}
\max\Big\{
&\bar\sigma(\cP_{E}\F^{\ast}\cP_{U}),\;
\bar\sigma(\cP_{E^\perp}\F^{\ast}\cP_{U}),\;
\bar\sigma(\cP_{E}\F^{\ast}\cP_{U^\perp}),\;
\bar\sigma(\cP_{E^\perp}\F^{\ast}\cP_{U^\perp})
\Big\}
\;\le\; \frac{\kappa_{\min}^{\ast}}{c_1\,\Lambda^{2}},
\end{align*}
where $c_1=\frac{16M}{M-6}$ and $\bar\sigma(\cdot)$ is the maximum singular value.
\end{assumption}

The next technical result from~\cite{meng2014learning} show that Restricted Fisher Eigenvalue and Structural Fisher Incoherence, together imply the RSC and SI condition.
This allows us to use the framework of decomposable regularization for the analysis of eq.\eqref{eq:sparselowrankproblem}.

\begin{proposition}[\textbf{RFE and SFI imply RSC and SI}, Lemma 2 and 3 in~\cite{meng2014learning}]
\label{prop:meng}
Let $\bOmega^{*}$ be the true marginal precision matrix and suppose Assumption~\ref{assum:rfe} and Assumption~\ref{assum:sfi} hold for $\bOmega^{*}$, and let $M>6$.
Restricted Strong Convexity is satisfied with tolerance function $\tau_{\cL} = 0$ and curvature parameter
$
  \kappa_{\cL}
  \;=\;
  \frac{M-2}{2(M-1)}\,\kappa_{\min}^{\star}$ for all $\bDelta \in \cC(E)\cup\cC(U)$ such that
$\|\bDelta\|_{\F^{*}}^{2} \;\le\; \frac{1}{2M^{2}}$.
Furthermore, Structural Incoherence is
satisfied for all $\bDelta_S\in\cC(E)$ and
$\bDelta_L\in\cC(U)$, such that
$  \max\Bigl\{\;\|\bDelta_S\|_{\F^{\star}}^{2},\;
                 \|\bDelta_L\|_{\F^{\star}}^{2}\Bigr\}
  \;\le\; \frac{1}{6M^{2}} \,$.
\end{proposition}

\subsection{Recovery Guarantees for the Confounder-Free Precision Matrix}

Next, we present recovery guarantees for the undirected graphical model.
In particular, we show that the confounder-free precision matrix $\S$ can be successfully recovered.
We want to point out that the results in~\cite{meng2014learning} provided recovery guarantees for $\bOmega=\S+\L$ only, making it difficult to disentangle de contributions of $\S$ and $\L$ in the final error bound.
In contrast, we provide an error bound for both the sparse and low-rank components.

\begin{theorem}[\textbf{Deterministic bound for $\S$ and $\L$}] \label{thm:SLdeterministic}
Let $\bOmega^{*}$ be the true marginal precision matrix and suppose Assumption~\ref{assum:rfe} and Assumption~\ref{assum:sfi} hold for $\bOmega^{*}$.
Let $\widehat{\bSigma}=\tfrac1n \X^\top \X$ denote the sample covariance matrix, and $\bSigma^*$ the true covariance matrix.
If the regularization parameters fulfill
\begin{equation*}
\lambda \ge 2\|\widehat{\bSigma}-\bSigma^{*}\|_{\infty}
\quad\text{and}\quad
\mu \ge 2\|\widehat{\bSigma}-\bSigma^{*}\|_{2},
\end{equation*}
then the following error bound holds for the estimators $\widehat{\S}$ and $\widehat{\L}$:
\begin{equation*}
\| \widehat{\S}-\S^{*}\|_{\infty} + \| \widehat{\L}-\L^{*}\|_{\infty}
\;\le\;
\frac{6}{\kappa_{\cL}}\,
\max\Bigl\{
\lambda\sqrt{s},
\; \mu\sqrt{r}
\Bigr\}.
\end{equation*}
where $s$ is the number of none-zero entries in $\S^*$, $r$ is the number of latent confounders, and
$\kappa_{\cL} := \frac{M-2}{2(M-1)}\,\kappa_{\min}^{*}$.
\end{theorem}
(Ommitted proofs can be found in Appendix~\ref{app:proofs}.)

\begin{remark}
Note that the above guarantee is in the form $\| \widehat{\S}-\S^{*}\|_{\infty} + \| \widehat{\L}-\L^{*}\|_{\infty} \le \epsilon$.
Since norms are non-negative, we have that $\| \widehat{\S}-\S^{*}\|_{\infty} \le \epsilon$ and $\| \widehat{\L}-\L^{*}\|_{\infty} \le \epsilon$.
\end{remark}

\cite{yang2013dirty} provided a general result for estimation with superposition of structurally constrained parameters.
We prove the theorem above by applying the general results of~\cite{yang2013dirty} to the specific problem of sparse and low-rank regularization in eq.\eqref{eq:sparselowrankproblem}.

Theorem~\ref{thm:SLdeterministic} provides a deterministic statement that does not consider the fact that data is random, and thus the sample covariance matrix $\widehat{\bSigma}$ is a random variable.
Next, we address this issue.

\begin{theorem}[\textbf{High-probability bound for $\S$ and $\L$}] \label{thm:SL}
Let $\bOmega^{*}$ be the true marginal precision matrix and suppose Assumption~\ref{assum:rfe} and Assumption~\ref{assum:sfi} hold for $\bOmega^{*}$.
Choose some constants
$C_1\ge 3/2$ and $C_{2}\ge 1$.
Assume the number of samples $n$ satisfies
$n \ge \max\{4C_{1}^{2}\log p,\; C_{2}^{2}p\}$.
Set the regularization parameters as
\begin{equation*}
\lambda = 160 C_{1}\,\bar\sigma^{*}\sqrt{\frac{\log p}{n}}
\qquad\text{and}\qquad
\mu = 16 C_{2}\,\rho^{*}\sqrt{\frac{p}{n}},
\end{equation*}
where $\bar\sigma^{*}=\max_{i}\bSigma^{*}_{i,i}$, 
$\rho^{*}=\| \bSigma^{*}\|_{2}$ and $p$ is the number of observed variables.
With probability at least
$1-4p^{-2(C_{1}-1)}-2\exp\!\bigl(-\tfrac{C_{2}^{2}p}{2}\bigr)$, we have
\begin{equation*}
\| \widehat{\S}-\S^{*}\|_{\infty} + \| \widehat{\L}-\L^{*}\|_{\infty}
\;\le\;
\max\Bigl\{
c_{1}\sqrt{\frac{s\log p}{n}},
\; c_{2}\sqrt{\frac{rp}{n}}
\Bigr\},
\end{equation*}
where $c_{1}=\dfrac{960}{\kappa_{L}}\bar\sigma^{*}C_1$,
$c_{2}=\dfrac{96}{\kappa_{L}}\rho^{*}C_2$ and $\kappa_{\cL}=\frac{M-2}{2(M-1)}\,\kappa_{\min}^{*}$, $s$ is the number of none-zero entries in $\S^*$, and $r$ is the number of latent confounders.
\end{theorem}

Later in Section~\ref{sec:algoandguarantees}, we motivate an algorithm that removes terminal nodes sequentially, one at a time, while having guarantees of recovering the true topological ordering.
At each iteration, our algorithm needs to solve eq.\eqref{eq:sparselowrankproblem}.
For our algorithm to have good statistical guarantees, we now show that if a (entrywise $\ell_\infty$ or spectral) norm deviation holds for matrices $\S$ and $\L$, then it also holds for all principal submatrices of $\S$ and $\L$.

\begin{claim} \label{claim:Snorms}
If $\lambda \ge 2\|\widehat{\bSigma}-\bSigma^{*}\|_{\infty}$ and $\mu \ge 2\|\widehat{\bSigma}-\bSigma^{*}\|_{2}$, then for all $I \subset [p]$, we have $\lambda \ge 2\|\widehat{\bSigma}_{I,I}-\bSigma^{*}_{I,I}\|_{\infty}$ and $\mu \ge 2\|\widehat{\bSigma}_{I,I}-\bSigma^{*}_{I,I}\|_{2}$.
\end{claim}

\subsection{Identifiability for SEMs over the Observed Variables}

Now, we turn our attention to directed graphical models, specifically to SEMs.
The previous section focused on the recovery of the confounder-free precision matrix $\S$.
Here we provide assumptions under which the weight matrix $\B$ can be successfully recovered from $\S$.

The next assumption is very relevant since our algorithm sequentially removes terminal nodes.
Thus, we require that our RSC and SI conditions hold for all induced subgraphs of $G$ and topological orderings.
By Proposition~\ref{prop:meng}, RFE and SFI imply the above conditions.

\begin{assumption}[\textbf{RFE and SFI for all induced subgraphs}] \label{assum:kappamin}
Let $\bOmega[m,\tau]$ denote the true precision matrix over the observed nodes in $V[m,\tau]$.
For all induced subgraphs $G[m,\tau]$, $m\in[p]$, and all topological orderings $\tau\in\T_{G}$, suppose Assumption~\ref{assum:rfe} and Assumption~\ref{assum:sfi} hold with constant $\kappa_{\min}^{*} > 0$.
That is, $\kappa_{\min}^{*}$ is the smallest constant among all $m\in[p]$, and $\tau\in\T_{G}$ satisfying Assumption~\ref{assum:rfe} and Assumption~\ref{assum:sfi}.
\end{assumption}

In the context of fully observed data (without latent confounders), the next identifiability condition is not only sufficient but also necessary for the identifiability of SEMs.
For instance, Lemma 1 in~\cite{ghoshal2018learning} showed that if the identifiability condition does not hold, then there exists exponentially many different SEMs that could have produced the given training data.

\begin{assumption}[\textbf{Identifiability Condition}, Assumption 1 in~\cite{ghoshal2018learning}] \label{assum:ident}
Let $(G,\B,\{\sigma_i^2\}_{i=1}^p))$ be a SEM and let $\S=(\I-\B)^{\top}\N_X^{-1}(\I-\B)$ be the related precision matrix, where $\N_X = \Diag\!\left(\{\sigma_i^2\}_{i=1}^p\right)$.
For all $(i,j)\in V[m,\tau]\times V[m,\tau]$, $m\in [p]$, and
$\tau\in \T_G$ such that $\phi_{G[m,\tau]}(i)=\varnothing$ and
$\phi_{G[m,\tau]}(j)\neq \varnothing$,
\begin{equation*}
\frac{1}{\sigma_i^{2}}
\;<\;
\frac{1}{\sigma_j^{2}}
\;+\;
\sum_{\,l\in \phi_{G[m,\tau]}(j)} \frac{\B_{l,j}^{2}}{\sigma_{l}^{2}}.
\end{equation*}
\end{assumption}

Given the above assumption, the next technical result from~\cite{ghoshal2018learning} allows us to recover both the true topological order of the graph $G$ as well as the edge weights $\B$, from the confounder-free precision matrix $\S$.

\begin{proposition}[\textbf{Recovery of $\B$ from $\S$}, Proposition 3 and 4 in~\cite{ghoshal2018learning}] \label{prop:ghoshal}
Let $(G,\B,\{\sigma_i^2\}_{i=1}^p))$ be a SEM and let $\S=(\I-\B)^{\top}\N_X^{-1}(\I-\B)$ be the related precision matrix, where $\N_X = \Diag\!\left(\{\sigma_i^2\}_{i=1}^p\right)$.
Under Assumption~\ref{assum:ident}, $i$ is a terminal node in $G$ if $i \in \argmin \S_{i,i}$.
Moreover, if $i$ is a terminal node in $G$, then $\B_{i,\bullet} = -\,\frac{\S_{i,\bullet}}{\S_{i,i}}$ and $\sigma_i^{2}=1/\S_{i,i}$.
\end{proposition}

The following assumption was inspired by Assumption 2 in~\cite{ghoshal2018learning}, and is essentially a stricter requirement than Assumption~\ref{assum:ident}, that allows to handle the randomness of the finite-sample data.

\begin{assumption}[\textbf{Finite Sample Identifiability Condition}] \label{assum:identfinite}
Let $(G,\B,\{\sigma_i^2\}_{i=1}^p))$ be a SEM and let $\S=(\I-\B)^{\top}\N_X^{-1}(\I-\B)$ be the related precision matrix, where $\N_X = \Diag\!\left(\{\sigma_i^2\}_{i=1}^p\right)$.
Suppose Assumption~\ref{assum:kappamin} hold, and let $\S[(m,\tau)]$ denote the confounder-free precision matrix over the observed nodes in $V[m,\tau]$.
Then we assume that:
\begin{description}
\item[i)] For all $(i,j)\in V[m,\tau]\times V[m,\tau]$, $m\in[p]$, and $\tau\in\T_{G}$
such that $\phi_{G[m,\tau]}(i)=\varnothing$ and $\phi_{G[m,\tau]}(j)\neq\varnothing$,
\begin{equation*}
\hspace{-0.125in}
\frac{1}{\sigma_i^2}
<
\frac{1}{\sigma_j^2}
+
\sum_{l\in \phi_{G[m,\tau]}(j)} \frac{\B_{l,j}^{2}}{\sigma_l^{2}}
\;-\;
\frac{12}{\kappa_{\cL}}\,
\max\Bigl\{
\lambda\sqrt{s},
\; \mu\sqrt{r}
\Bigr\},
\end{equation*}
\item[ii)] \(\displaystyle
\min\big\{\,|(\S[m,\tau])_{i,j}|:\,(\S[m,\tau])_{i,j}\neq 0,\ (i,j)\in V[m,\tau]\times V[m,\tau],\ m\in[p],\ \tau\in\T_G \big\} 
\;>\; 
\frac{6}{\kappa_{\cL}}\,
\max\Bigl\{
\lambda\sqrt{s},
\; \mu\sqrt{r}
\Bigr\},
\)
\end{description}
where $s$ is the number of none-zero entries in $\S^*$, $r$ is the number of latent confounders, and
$\kappa_{\cL} = \frac{M-2}{2(M-1)}\,\kappa_{\min}^{*}$ with $M\;>\;6$.
\end{assumption}

\subsection{Algorithm and Recovery Guarantees for SEMs over the Observed Variables} \label{sec:algoandguarantees}

Here we present our algorithm and show its theoretical guarantees.
First, we motivate an algorithm that recovers the true topological ordering of $G$, by detecting terminal nodes sequentially, one at a time.
More specifically, at each iteration we solve the (undirected graphical model) optimization problem in eq.\eqref{eq:sparselowrankproblem} in order to recover $\S$.
By Proposition~\ref{prop:ghoshal}, the node $i$ with smallest value in the diagonal ($\S_{i,i}$) is a terminal node.
We then recover the weights $\B_{i,\bullet}$ and noise variance $\sigma_i^2$ from $\S$, also by using Proposition~\ref{prop:ghoshal}.
Algorithm~\ref{alg:getdag} describes this process in detail.

\begin{algorithm}
\caption{Learning a SEM over the observed variables}
\label{alg:getdag}
\begin{algorithmic}[1]
  \STATE \textbf{Input:} Data matrix $\X\in\R^{n\times p}$, regularization parameters $\lambda>0$ and $\mu>0$
  \STATE $\widehat{\B}\gets \zero \in \R^{p \times p}$;
  vars $\gets [p]$;
  $\widehat{\bSigma} \gets \X^\top \X / n$
  \FOR{$t=1,\dots,p-1$}
    \STATE Solve eq. \eqref{eq:sparselowrankproblem} to obtain $\widehat{\S}$ and $\widehat{\L}$ using $\widehat{\bSigma}$;
    $i \gets \argmin \widehat{\S}_{i,i}$ \hfill 
    \STATE \textbf{for} each index $j$ and value $v$ in vars \textbf{do} $\widehat{\B}_{\text{vars}[i],\,v} \gets -\widehat{\S}_{i,j}/\widehat{\S}_{i,i}$
    \STATE $\widehat{\B}_{\text{vars}[i],\text{vars}[i]} \gets 0$;
    $\widehat{\sigma}_{\text{vars}[i]}^{2}=1/\widehat{\S}_{i,i}$;
    vars $\gets$ vars $\setminus \{i\}$;
    $\X \gets$ $\X_{\bullet,-i}$;
    $\widehat{\bSigma} \gets \X^\top \X / n$
  \ENDFOR
  \RETURN $(\widehat{G},\widehat{\B},\{\widehat{\sigma}_i^2\}_{i=1}^p)$ where $\widehat{G}=([p],\widehat{E})$ and $\widehat{E}=\supp(\widehat{\B})$
\end{algorithmic}
\end{algorithm}

Next, we show the theoretical guarantees for Algorithm~\ref{alg:getdag}.
We show that the value of every entry in $\widehat{\B}$ is close to those of $\B^{*}$.
In order to do so, we use the entrywise $\ell_\infty$ norm in our deviation bound.
Furthermore, we show that the recovered edges $\supp(\widehat{\B})$ are the same as the true edges $\supp(\B^*)$.

\begin{theorem}[\textbf{Deterministic recovery guarantee for $\B$}] \label{thm:Bdeterministic}
Suppose Assumption~\ref{assum:kappamin} and Assumption~\ref{assum:identfinite} hold.
Let $(G^*, \B^*,\,\{{\sigma^*_i}^{2}\})$ be the true SEM and let $\S^*=(\I-\B^*)^{\top}{\N^*_X}^{-1}(\I-\B^*)$ be the related precision matrix, where $\N^*_X = \Diag\!\left(\{{\sigma^*_i}^2\}_{i=1}^p\right)$.
If the regularization parameters satisfy
$\lambda \ge 2\|\widehat{\bSigma}-\bSigma^{*}\|_{\infty}$ and 
$\mu \ge 2\|\widehat{\bSigma}-\bSigma^{*}\|_{2}$.
Algorithm ~\ref{alg:getdag} returns an estimation
$\widehat{\B}$ such that
\begin{equation*}
\| \widehat{\B}-\B^{*}\|_{\infty}
\;\le\; c\,(1+\|\B^*\|_\infty) \, m,
\end{equation*}
where $m = \frac{6}{\kappa_{\cL}}\,
\max\Bigl\{
\lambda\sqrt{s},
\; \mu\sqrt{r}
\Bigr\}$, $s$ is the number of none-zero entries in $\S^*$, $r$ is the number of latent confounders,
$\kappa_{\cL} = \frac{M-2}{2(M-1)}\,\kappa_{\min}^{*}$, $M\;>\;6$
and $c \ge \max_{i\in[p]} \sigma_i^{2}/\bigl(1-m\,\sigma_i^{2}\bigr) > 0$.
Furthermore, Algorithm~\ref{alg:getdag} correctly recovers the edges of the true SEM the true SEM $(G^{*}, \B^{*},\,\{{\sigma^*_i}^{2}\})$, i.e., $\supp(\widehat{\B})=\supp(\B^*)$.
\end{theorem}

Theorem~\ref{thm:Bdeterministic} provides a deterministic statement that does not consider the fact that the training data is random, which implies that the sample covariance matrix $\widehat{\bSigma}$ is also random.
In what follows, we address this.

The following theorem shows that if the number of samples $n$ fulfills $n \gtrsim \frac{1}{\epsilon^2}\max\left\{s\log p,\ r p\right\}$, then Algorithm~\ref{alg:getdag} successfully recovers the weights in $\B$ as well as the edges in the graph $G$.

\begin{theorem}[\textbf{High-probability recovery guarantee for $\B$}] \label{thm:B}
Suppose Assumption~\ref{assum:kappamin} and Assumption~\ref{assum:identfinite} hold.
Let $(G^{*}, \B^{*},\,\{{\sigma^*_i}^{2}\})$ be the true SEM and let $\S^*=(\I-\B^*)^{\top}{\N^*_X}^{-1}(\I-\B^*)$ be the related precision matrix, where $\N^*_X = \Diag\!\left(\{{\sigma^*_i}^2\}_{i=1}^p\right)$.
Choose some constants
$C_1\ge 3/2$ and $C_{2}\ge 1$.
Assume the number of samples $n$ satisfies
\begin{align*}
n &\ge \max\left\{4C_{1}^{2}\log p, \; C_{2}^{2}p, 
c^{2}(1+\|\B^*\|_\infty)^{2}/\epsilon^{2}
\max\{
c_1^{2}s\log p,
\; c_2^{2}rp
\}
\right\},
\end{align*}
Set the regularization parameters as
\begin{equation*}
\lambda = 160 C_{1}\,\bar\sigma^{*}\sqrt{\frac{\log p}{n}}
\qquad\text{and}\qquad
\mu = 16 C_{2}\,\rho^{*}\sqrt{\frac{p}{n}},
\end{equation*}
where $\bar\sigma^{*}=\max_{i}\bSigma^{*}_{i,i}$,
$\rho^{*}=\| \bSigma^{*}\|_{2}$, $s$ is the number of non-zero entries in $\S^*$, $p$ is the number of observed variables, $r$ is the number of latent confounders, $c>0$ is defined as in Theorem~\ref{thm:Bdeterministic}, $c_{1}=\dfrac{960}{\kappa_{L}}\bar\sigma^{*}C_1$,
$c_{2}=\dfrac{96}{\kappa_{L}}\rho^{*}C_2$, $\kappa_{\cL} = \frac{M-2}{2(M-1)}\,\kappa_{\min}^{*}$ and $M\;>\;6$.
Algorithm ~\ref{alg:getdag} returns an estimation
$\widehat{\B}$ such that
\[
\| \widehat{\B}-\B^{*}\|_{\infty} \le \epsilon,
\]
with probability at least $1-4p^{-2(C_{1}-1)}-2\exp\!\bigl(-\tfrac{C_{2}^{2}p}{2}\bigr)$.
Furthermore, Algorithm~\ref{alg:getdag} correctly recovers the edges of the true SEM $(G^{*}, \B^{*},\,\{{\sigma^*_i}^{2}\})$, i.e., $\supp(\widehat{\B})=\supp(\B^*)$.
\end{theorem}

\section{Experimental Validation}

\begin{figure*}
\centering
\begin{tabular}{cc}
\includegraphics[width=0.5\linewidth]{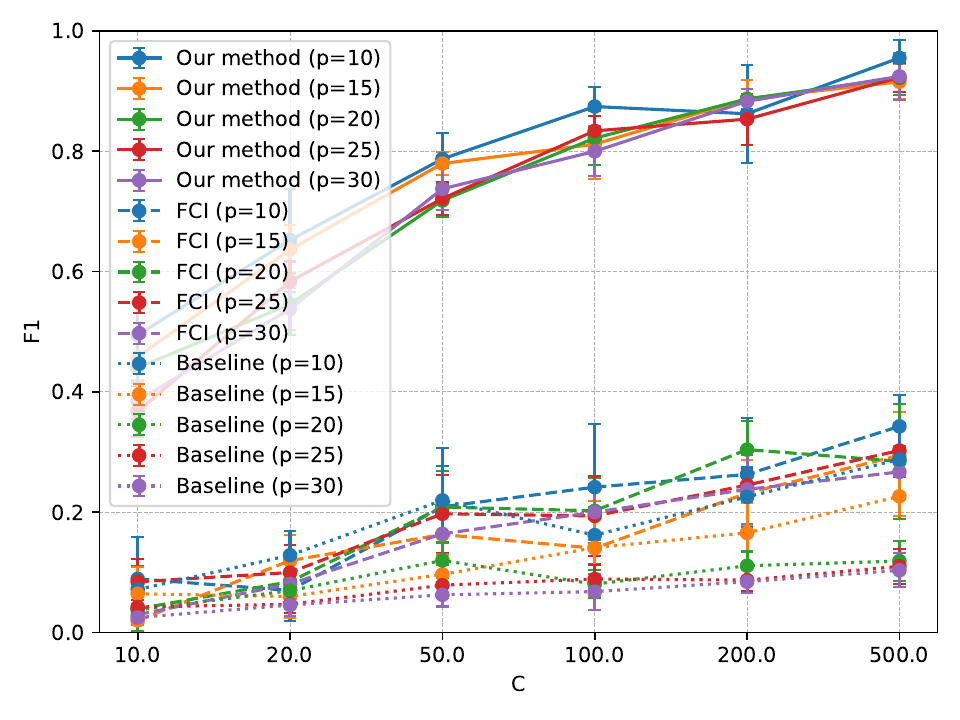} & \includegraphics[width=0.5\linewidth]{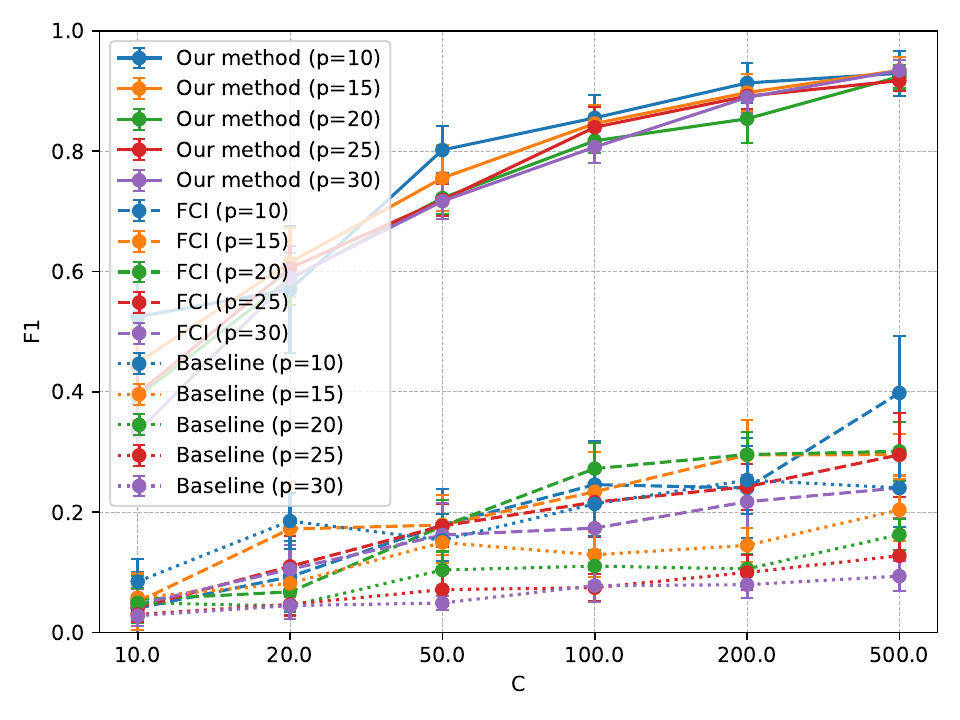} \\
(a) Gaussian distribution & (b) Uniform distribution \\
\includegraphics[width=0.5\linewidth]{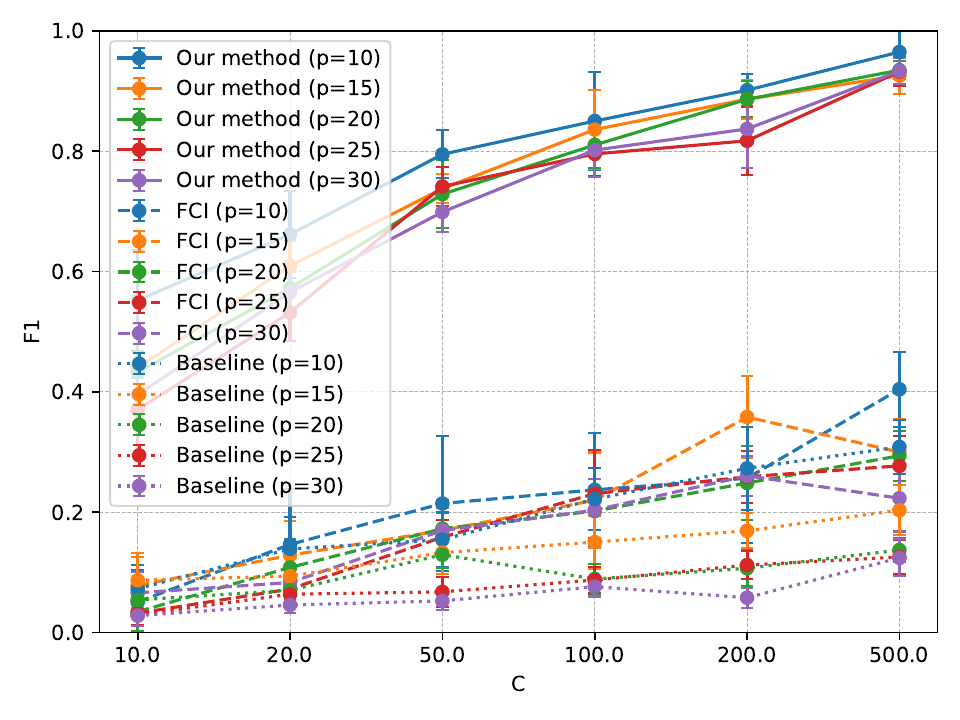} & \includegraphics[width=0.5\linewidth]{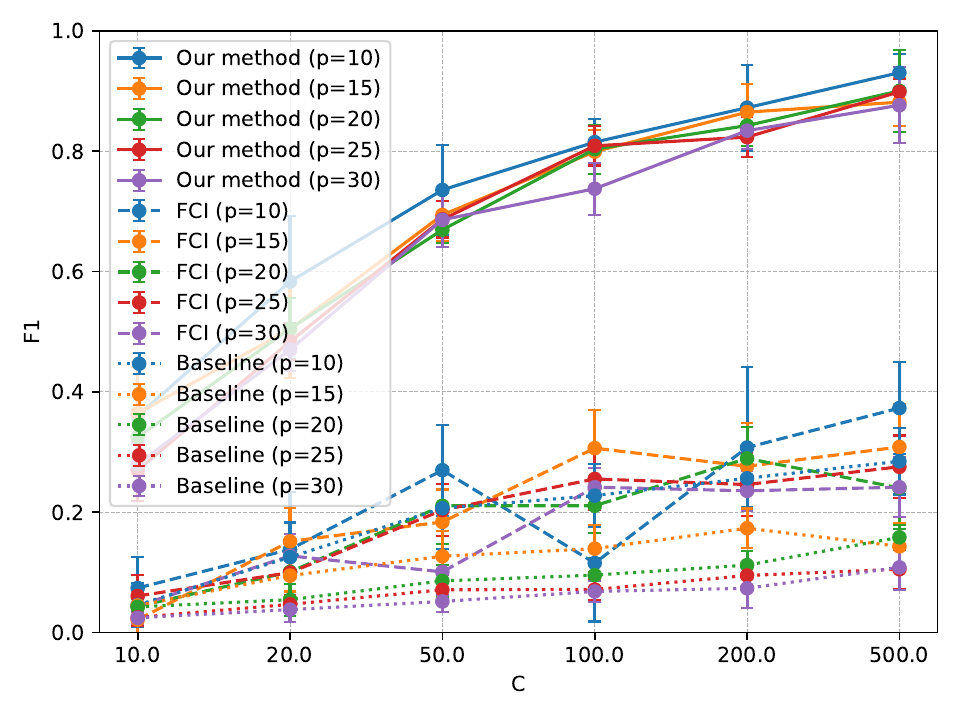} \\
(c) Mixture & (d) Heavy-tailed distribution \\
\end{tabular}
\vspace{-0.1in}
\caption{F1 score (predicted edges versus true edges) for various experimental regimes: (a) Gaussian noise, (b) uniformly-distributed noise, (c) noise following a mixture distribution between a Gaussian and a uniform distribution, and (d) $t$-Student distributed noise.
Our method performs better than the baseline method of~\cite{ghoshal2018learning} for various sample sizes and noise types.
Moreover, when the number of samples is large enough, i.e., $n=C\log p$ for $C=500$, our method obtains a very good F1 score.}
\label{fig:all}
\end{figure*}

In this section, we validate our main theoretical result in Theorem~\ref{thm:Bdeterministic}.
Our experiments consider noise $\noise_X$ and $\noise_Z$ from various distributions, such as Gaussian, uniform, a mixture distribution, as well as heavy-tailed distribution.
We use a mixture distribution of a Gaussian (50\%) and a uniform distribution (50\%).
For the heavy-tailed distribution, we use a $t$-Student distribution with $3$ degrees of freedom.

We consider $r=2$ latent confounders and different number of observed variables $p \in \{10,15,20,25,30\}$.
The topological ordering of the graph is created randomly.
We create an edge set $E$ randomly.
Every observed variable can be affected by at most $2$ observed variables and $1$ latent confounder.
Every latent confounder can be affected by at most $1$ latent confounder.
But a latent confounder can affect several observed variables.
After we decide that edge $(i,j)\in E$, we set $\B_{i,j}=-1$ or $\B_{i,j}=+1$ with equal probability.
Similarly, $\C_{i,j}=-1/4$ or $\C_{i,j}=+1/4$ is set with equal probability.
Finally, $\D_{i,j}=-1/4$ or $\D_{i,j}=+1/4$ is also set with equal probability.

For every node $i$, the noise standard deviation $\sigma_i$ was generated uniformly at random from $[0.08,0.12]$.
We then generate $n$ samples where $n=C\log p$ for $C \in \{10,20,50,100,200,500\}$.

We run our Algorithm~\ref{alg:getdag} with $\lambda=0.01\sqrt{(\log p)/n}$, $\mu=0.005\sqrt{(\log p)/n}$.
We used the method of~\cite{ghoshal2018learning} as a baseline which does not take into account the latent confounders.
We also compared to fast causal inference (FCI)~\cite{spirtes2000causation,colombo_learning_2012} which outputs partial ancestral graphs.
Unfortunately, other methods do not apply for our problem.

We use the F1 score as the metric, computed between the predicted edges and the true edges.
We repeat this produce 10 times and report the average F1 score as well as 95\%-confidence error bars.

Figure~\ref{fig:all} shows that our method performs better than the baseline method of~\cite{ghoshal2018learning} for various sample sizes and noise types.
Moreover, when the number of samples is large enough, i.e., $n=C\log p$ for $C=500$, our method obtains a very good F1 score.

\section{Concluding Remarks}

Our results open several interesting research questions.
Although our model provides a theoretical guarantee for recovering the causal edges between observed variables, it does not pay much attention to the causal edges between latent confounders and the causal edges from latent confounders to observed variables.
Future work could focus on recovery guarantees for those aspects.
In addition, our model is based on a sparse plus low-rank structure.
Another research direction is enable different structural conditions.

\bibliographystyle{plain}
\bibliography{refs}

@inproceedings{meng2014learning,
  title={Learning latent variable {G}aussian graphical models},
  author={Meng, Zhaoshi and Eriksson, Brian and Hero, Al},
  booktitle={International Conference on Machine Learning},
  pages={1269--1277},
  year={2014}
}

@inproceedings{ghoshal2018learning,
  title={Learning linear structural equation models in polynomial time and sample complexity},
  author={Ghoshal, Asish and Honorio, Jean},
  booktitle={Artificial Intelligence and Statistics},
  pages={1466--1475},
  year={2018}
}

@inproceedings{ng2024score,
  title={Score-based causal discovery of latent variable causal models},
  author={Ng, Ignavier and Dong, Xinshuai and Dai, Haoyue and Huang, Biwei and Spirtes, Peter and Zhang, Kun},
  booktitle={International Conference on Machine Learning},
  year={2024}
}

@article{huang2022latent,
  title={Latent hierarchical causal structure discovery with rank constraints},
  author={Huang, Biwei and Low, Charles Jia Han and Xie, Feng and Glymour, Clark and Zhang, Kun},
  journal={Neural Information Processing Systems},
  volume={35},
  pages={5549--5561},
  year={2022}
}

@article{Silva2003-SILLMM,
	author = {Ricardo Silva and Richard Scheines and Clark Glymour and Peter Spirtes},
	title = {Learning Measurement Models for Unobserved Variables},
    journal = {Uncertainty in Artificial Intelligence},
	year = {2003}
}

@inproceedings{chandrasekaran2010latent,
  title={Latent variable graphical model selection via convex optimization},
  author={Chandrasekaran, Venkat and Parrilo, Pablo A and Willsky, Alan S},
  booktitle={Allerton Conference on Communication, Control, and Computing},
  pages={1610--1613},
  year={2010}
}

@article{negahban2012unified,
  title={A unified framework for high-dimensional analysis of {M}-estimators with decomposable regularizers},
  author={Negahban, Sahand N and Ravikumar, Pradeep and Wainwright, Martin J and Yu, Bin},
  journal={Statistical Science},
  year={2012}
}

@article{yang2013dirty,
  title={Dirty statistical models},
  author={Yang, Eunho and Ravikumar, Pradeep K},
  journal={Neural Information Processing Systems},
  volume={26},
  year={2013}
}

@book{spirtes2000causation,
  title={Causation, prediction, and search},
  author={Spirtes, Peter and Glymour, Clark N and Scheines, Richard},
  year={2000},
  publisher={MIT press}
}

@article{ravikumar2011high,
  title={High-dimensional covariance estimation by minimizing $\ell_1$-penalized log-determinant divergence},
  author={Ravikumar, Pradeep and Wainwright, Martin J and Raskutti, Garvesh and Yu, Bin},
  journal={Electronic Journal of Statistics},
  year={2011}
}

@article{bickel2009simultaneous,
  title={Simultaneous analysis of Lasso and Dantzig selector},
  author={Bickel, Peter J and Ritov, Ya’acov and Tsybakov, Alexandre B},
  journal={Annals of Statistics},
  year={2009}
}

@inproceedings{kakade2010learning,
  title={Learning exponential families in high-dimensions: Strong convexity and sparsity},
  author={Kakade, Sham and Shamir, Ohad and Sindharan, Karthik and Tewari, Ambuj},
  booktitle={Artificial Intelligence and Statistics},
  pages={381--388},
  year={2010}
}

@article{colombo_learning_2012,
	title = {Learning high-dimensional directed acyclic graphs with latent and selection variables},
	volume = {40},
	number = {1},
	journal = {The Annals of Statistics},
	author = {Colombo, Diego and Maathuis, Marloes H. and Kalisch, Markus and Richardson, Thomas S.},
	month = feb,
	year = {2012},
}

@article{silva06a,
  author  = {Ricardo Silva and Richard Scheine and Clark Glymour and Peter Spirtes},
  title   = {Learning the Structure of Linear Latent Variable Models},
  journal = {Journal of Machine Learning Research},
  year    = {2006},
  volume  = {7},
  number  = {8},
  pages   = {191--246},
}

@InProceedings{xie22a,
  title = 	 {Identification of Linear Non-{G}aussian Latent Hierarchical Structure},
  author =       {Xie, Feng and Huang, Biwei and Chen, Zhengming and He, Yangbo and Geng, Zhi and Zhang, Kun},
  booktitle = 	 {International Conference on Machine Learning},
  pages = 	 {24370--24387},
  year = 	 {2022},
  volume = 	 {162},
}

@inproceedings{cui2018,
  author       = {Ruifei Cui and
                  Perry Groot and
                  Moritz Schauer and
                  Tom Heskes},
  title        = {Learning the Causal Structure of Copula Models with Latent Variables},
  booktitle    = {Uncertainty in Artificial Intelligence},
  pages        = {188--197},
  year         = {2018},
}

@inproceedings{zheng2018,
 author = {Zheng, Xun and Aragam, Bryon and Ravikumar, Pradeep K and Xing, Eric},
 booktitle = {Neural Information Processing Systems},
 pages = {},
 title = {{DAGs} with {NO} {TEARS}: Continuous Optimization for Structure Learning},
 volume = {31},
 year = {2018}
}

@inproceedings{ng2020,
 author = {Ng, Ignavier and Ghassami, AmirEmad and Zhang, Kun},
 booktitle = {Neural Information Processing Systems},
 pages = {17943--17954},
 title = {On the Role of Sparsity and {DAG} Constraints for Learning Linear {DAGs}},
 volume = {33},
 year = {2020}
}

\clearpage
\appendix

\section{Covariance and Precision Matrix for SEMs with Latent Confounders} \label{app:covarianceprecision}

From eq.\eqref{eq:sem} we have that
\begin{align*}
\left( \I - \begin{bmatrix}
\B & \C \\
\zero & \D
\end{bmatrix} \right) \begin{bmatrix} \x \\ \z \end{bmatrix}
=
\begin{bmatrix} \noise_X \\ \noise_Z \end{bmatrix},
\end{align*}
and thus
\begin{align*}
\begin{bmatrix} \x \\ \z \end{bmatrix}
=
\left( \I - \begin{bmatrix}
\B & \C \\
\zero & \D
\end{bmatrix} \right)^{-1} \begin{bmatrix} \noise_X \\ \noise_Z \end{bmatrix} .
\end{align*}
By eq.\eqref{eq:Sigmafull} we have that
\begin{align*}
\bSigma_{\full} & = \E\left[\begin{bmatrix} \x \\ \z \end{bmatrix} \begin{bmatrix} \x \\ \z \end{bmatrix}^\top\right] \\
& = \left( \I - \begin{bmatrix}
\B & \C \\
\zero & \D
\end{bmatrix} \right)^{-1} \E\left[ \begin{bmatrix} \noise_X \\ \noise_Z \end{bmatrix} \begin{bmatrix} \noise_X \\ \noise_Z \end{bmatrix}^\top \right] \left( \I - \begin{bmatrix}
\B & \C \\
\zero & \D
\end{bmatrix} \right)^{-\top} \\
& = \left( \I - \begin{bmatrix}
\B & \C \\
\zero & \D
\end{bmatrix} \right)^{-1} \begin{bmatrix}
\N_X & 0\\
0 & \N_Z
\end{bmatrix} \left( \I - \begin{bmatrix}
\B & \C \\
\zero & \D
\end{bmatrix} \right)^{-\top} ,
\end{align*}
where $\N_X = \Diag\!\left(\{\sigma_i^2\}_{i=1}^p\right)$ and $\N_Z = \Diag\!\left(\{\sigma_i^2\}_{i=p+1}^{p+r}\right)$.
Now, by eq.\eqref{eq:Jfull} we have
\begin{align*}
\J_\full & = \bSigma_\full^{-1} \\
& = \left(\I -
\begin{bmatrix}
\B & \C\\
\zero & \D
\end{bmatrix}
\right)^{\top}
\begin{bmatrix}
\N_X^{-1} & 0\\
0 & \N_Z^{-1}
\end{bmatrix}
\left(\I -
\begin{bmatrix}
\B & \C\\
\zero & \D
\end{bmatrix}
\right) \\
& = \begin{bmatrix}
\I-\B & -\C\\
\zero & \I-\D
\end{bmatrix}^{\top}
\begin{bmatrix}
\N_X^{-1} & 0\\
0 & \N_Z^{-1}
\end{bmatrix}
\begin{bmatrix}
\I-\B & -\C\\
\zero & \I-\D
\end{bmatrix} \\
& = \begin{bmatrix}
(\I-\B)^{\top} \N_X^{-1} (\I-\B) & -(\I-\B)^{\top} \N_X^{-1} \C \\
- \C^{\top} \N_X^{-1} (\I-\B) & \C^{\top} \N_X^{-1} \C + (\I-\D)^{\top} \N_Z^{-1} (\I-\D)
\end{bmatrix}
\end{align*}
From the above and since $\J_\full =
\left[\begin{matrix} \S & \J_{X,Z} \\ \J_{X,Z}^\top & \J_{Z,Z}  \end{matrix}\right]$, we conclude that $\S = (\I-\B)^{\top} \N_X^{-1}(\I-\B)$.
Finally, we use eq.(1) in~\cite{meng2014learning} to obtain the expression for $\bOmega = \bSigma^{-1} = \S+\L$ where $\L=-\J_{X,Z}\,\J_{Z,Z}^{-1}\,\J_{X,Z}^\top$.

\section{Proofs} \label{app:proofs}

Here we present the proofs for the theorems and lemmas in our main text.

\subsection{Proof of Theorem~\ref{thm:SLdeterministic}}

\begin{proof}
By Assumption~\ref{assum:rfe} and Assumption~\ref{assum:sfi} for $\bOmega^{*}$, by Proposition~\ref{prop:meng} we have that RSC and SI conditions hold.
The proof of Theorem~\ref{thm:SLdeterministic} refers to the proof of Theorem 1 and Corollary 4 in~\cite{yang2013dirty}. They pointed out that under RSC and SI condition, if $\lambda_{\alpha} \ge 2\,\cR_{\alpha}^{\ast}\!\left(\nabla_{\bOmega_{\alpha}} \cL(\bOmega^{\ast}; Z_{1}^{n})\right)
$, then
\begin{equation*}
\sum_{\alpha\in I} \| \widehat{\bDelta}_{\alpha} \|
\;\le\;
\frac{|I|}{\bar{\kappa}}
\left(\frac{3}{2}\,\Phi + \sqrt{\bar{\kappa}\,\tau_{\cL}}\right)
\end{equation*}
where $\Phi = \max_{\alpha\in I}\,\lambda_{\alpha}\,\Psi_{\alpha}\!\big(\overline{\M}_{\alpha}\big)$ and $\cR^{\ast}$ is the dual norm of the norm. The dual norm of the entrywise $\ell_{1}$ norm is the entrywise $\ell_\infty$ norm, and the dual norm of the nuclear norm is the spectral norm.
Note that $\nabla \cL(\bOmega;\X) = \widehat{\bSigma} - \bOmega^{-1}= \widehat{\bSigma} - \bSigma^{*}$.
Thus, when we set $\lambda \ge 2\|\widehat{\bSigma}-\bSigma^{*}\|_{\infty}$ and $\mu \ge 2\|\widehat{\bSigma}-\bSigma^{*}\|_{2}$, we obtain an error bound.

Next, we calculate the error bound.
In our particular problem, $I=\{S,L\}$, $\tau_{\cL}=0$, $\bar{\kappa}=\frac{\kappa_{\cL}}{2}$, and thus we can get $\| \widehat{\bDelta}_{S} \| + \| \widehat{\bDelta}_{L} \| \;\le\; \frac{6}{\kappa_{\cL}}\max\Bigl\{ \lambda\Psi_{S}\!\big(\overline{\M}_{S}\big),\mu \Psi_{L}\!\big(\overline{\M}_{L}\big)\Bigl\}$. Next, we get:
\begin{align*}
\Psi_{S}\!\big(\overline{\M}_{S}\big)
&= \sup_{\bDelta \in \overline{\M}(E)\setminus\{0\}}
   \frac{\| \bDelta\|_{1}}{\| \bDelta\|_{F}}
   \le \sqrt{s},\\
\Psi_{L}\!\big(\overline{\M}_{L}\big)
&= \sup_{\bDelta \in \overline{\M}(U)\setminus\{0\}}
   \frac{\| \bDelta\|_{*}}{\| \bDelta\|_{F}}
   \le \sqrt{r}.
\end{align*}

Since we are using the Frobenius norm in RSC and SI, we can get $\| \widehat{\bDelta}_{S}\|_{F}  + \| \widehat{\bDelta}_{L} \|_{F}  \;\le\; \frac{6}{\kappa_{\cL}}\max\Bigl\{ \lambda\sqrt{s},\mu \sqrt{r}\}$.
Since the entrywise $\ell_\infty$ norm is no greater than the Frobenius norm, we get
\begin{equation*}
\| \widehat{\bDelta}_{S}\|_{\infty}  + \| \widehat{\bDelta}_{L} \|_{\infty}  \;\le\; \frac{6}{\kappa_{\cL}}\max\Bigl\{ \lambda\sqrt{s},\mu \sqrt{r}\}.
\end{equation*}
\end{proof}

\subsection{Proof of Theorem~\ref{thm:SL}}

\begin{proof}
Theorem~\ref{thm:SLdeterministic} is a deterministic theorem, and the error bound and the regularization parameters $\lambda$ and $\mu$ depend on the randomness in the data.
By Assumption~\ref{assum:rfe} and Assumption~\ref{assum:sfi} for $\bOmega^{*}$, by Proposition~\ref{prop:meng} we have that RSC and SI conditions hold.
To prove Theorem~\ref{thm:Bdeterministic}, we only need to prove that $\lambda \ge 2\|\widehat{\bSigma}-\bSigma^{*}\|_{\infty}$ and $\mu \ge 2\|\widehat{\bSigma}-\bSigma^{*}\|_{2}$ holds with high probability. 

According to Lemma 5 in~\cite{meng2014learning} or Lemma 1~\cite{ravikumar2011high}, we can get:
\begin{equation*}
P\!\left\{\left\| \widehat{\bSigma}-\bSigma^\ast \right\|_\infty \le \tfrac{1}{2}\lambda \right\}
\;\ge\; 1-4\,p^{-2(C_1-1)} .
\end{equation*}
where $C_1 > 1$ and $n$ fulfills $n \ge 4C_1^{2}\log p$.
According to Lemma 6 in~\cite{meng2014learning} or Lemma 5.4 in~\cite{chandrasekaran2010latent}, we can get:
\begin{equation*}
P\!\left\{\left\| \widehat{\bSigma}-\bSigma^\ast \right\|_{2} \le \tfrac{1}{2}\mu \right\}
\;\ge\; 1 - 2\exp\!\left(-\frac{C_{2}^{\,2} p}{2}\right).
\end{equation*}
where $C_2 \ge 1$ and $n$ should be satisfy $n \ge C_2^{2} p$.
Thus, with probability at least $1-4p^{-2(C_{1}-1)}-2\exp\!\bigl(-\tfrac{C_{2}^{2}p}{2}\bigr)$, $\lambda \ge 2\|\widehat{\bSigma}-\bSigma^{*}\|_{\infty}$ and $\mu \ge 2\|\widehat{\bSigma}-\bSigma^{*}\|_{2}$ are both true.
\end{proof}

\subsection{Proof of Claim~\ref{claim:Snorms}}

\begin{proof}
Let $\bDelta = \widehat{\bSigma}-\bSigma^{*}$ and note that $\Delta$ is symmetric since both $\widehat{\bSigma}$ and $\bSigma^{*}$ are symmetric.
We need to show that $\lambda \ge 2\|\bDelta\|_{\infty}$ and $\mu \ge 2\|\bDelta\|_{2}$ hold for $\bDelta$, then for all principal submatrices $\bDelta_{I,I}$ where $I\subset [p]$, we have $\lambda \ge 2\|\bDelta_{I,I}\|_{\infty}$ and $\mu \ge 2\|\bDelta_{I,I}\|_{2}$.
The claim follows straightforwardly by properties of the entrywise $\ell_\infty$ and spectral norms, for symmetric principal submatrices.
\end{proof}

\subsection{Proof of Theorem~\ref{thm:Bdeterministic}}

\begin{proof}
According to Lemma~\ref{claim:Snorms}, after removing the terminal node, the error bound between the update precision matrix and the true precision matrix is always established, which implies $\| \widehat{\S}-\S^{*}\|_{\infty}
\;\le\;
\frac{6}{\kappa_{\cL}}\,
\max\Bigl\{
\lambda\sqrt{s},
\; \mu\sqrt{r}
\Bigr\}$.
Let $m = \frac{6}{\kappa_{\cL}}\,
\max\Bigl\{
\lambda\sqrt{s},
\; \mu\sqrt{r}
\Bigr\}$, $\varepsilon_{i,j}=\S^{\ast}_{i,j}-\widehat{\S}_{i,j}$, we get $\left|\varepsilon_{i,j}\right| \le m$ for all $i$ and $j$.
For any $i\neq j$,
\begin{align*}
\bigl|\widehat{\B}_{i,j}-\B^{\ast}_{i,j}\bigr|
&=\left|\frac{\widehat{\S}_{i,j}}{\widehat{\S}_{i,i}}
      -\frac{\S^{\ast}_{i,j}}{\S^{\ast}_{i,i}}\right|\\
&=\left|\frac{\S^{\ast}_{i,i}\bigl(\S^{\ast}_{i,j}-\varepsilon_{i,j}\bigr)
            -\bigl(\S^{\ast}_{i,i}-\varepsilon_{i,i}\bigr)\S^{\ast}_{i,j}}
            {\bigl(\S^{\ast}_{i,i}-\varepsilon_{i,i}\bigr)\S^{\ast}_{i,i}}\right|\\
&=\left|\frac{\varepsilon_{i,i}\,\S^{\ast}_{i,j}-\S^{\ast}_{i,i}\,\varepsilon_{i,j}}
            {\bigl(\S^{\ast}_{i,i}-\varepsilon_{i,i}\bigr)\S^{\ast}_{i,i}}\right|\\
&=\left|\frac{\varepsilon_{i,j}-\B^{\ast}_{i,j}\,\varepsilon_{i,i}}
            {\,\S^{\ast}_{i,i}-\varepsilon_{i,i}\,}\right|\\
& =\left|\frac{\varepsilon_{i,j}-\B^{\ast}_{i,j}\,\varepsilon_{i,i}}
            {\,1/\sigma_i^{2}-\varepsilon_{i,i}\,}\right|\\
&\le \frac{m(1+|\B^{\ast}_{i,j}|)}
          {\bigl|1/\sigma_i^{2}-\varepsilon_{i,i}\bigr|}\\
&  \le \frac{m(1+|\B^{\ast}_{i,j}|)}
          {1/\sigma_i^{2}-m}\\
&\;\le\; c\,m\,(1+|\B^{\ast}_{i,j}|) \\
&\;\le\; c\,m\,(1+\|\B^*\|_\infty)
\end{align*}

In the proof, we use $\S_{i,i}^{\ast} = 1/\sigma_i^{2}$ and $\S_{i,j}^{\ast} = -\B_{i,j}^{\ast}/\sigma_i^{2}$. In Assumption~\ref{assum:identfinite}, we use $1/\sigma_i^{2} > m \ge \varepsilon_{i,i}$. Then, 
every time a terminal node is removed, $\bigl|\widehat{\B}_{i,j}-\B^{\ast}_{i,j}\bigr| \le c\,m\,(1+\|\B^*\|_\infty) $ is correct for any $i\neq j$.
Therefore, we can get $\| \widehat{\B}-\B^{\ast}\|_{\infty}
\;\le\; c\,m\,(1+\|\B^*\|_\infty)$.

Regarding the correct topological ordering, by Assumption~\ref{assum:identfinite}, we know that for all terminal nodes $i$ and non-terminal nodes $j$:
\begin{equation*}
\frac{1}{\sigma_i^2}
\;<\;
\frac{1}{\sigma_j^2}
\;+\;
\sum_{l\in \phi_{G[m,\tau]}(j)} \frac{\B_{l,j}^{2}}{\sigma_l^{2}}
\;-\; 
\frac{12}{\kappa_{\cL}}\,
\max\Bigl\{
\lambda\sqrt{s},
\; \mu\sqrt{r}
\Bigr\}.
\end{equation*}
Thus, we can get:
\begin{equation*}
\frac{1}{\sigma_i^2} + \frac{6}{\kappa_{\cL}}\,
\max\Bigl\{
\lambda\sqrt{s},
\; \mu\sqrt{r}
\Bigr\}
\;<\;
\frac{1}{\sigma_j^2}
\;+\;
\sum_{l\in \phi_{G[m,\tau]}(j)} \frac{\B_{l,j}^{2}}{\sigma_l^{2}}
\;-\; 
\frac{6}{\kappa_{\cL}}\,
\max\Bigl\{
\lambda\sqrt{s},
\; \mu\sqrt{r}
\Bigr\}.
\end{equation*}
We know that $\| \widehat{\S}-\S\|_{\infty} \le \frac{6}{\kappa_{\cL}}\,
\max\Bigl\{
\lambda\sqrt{s},
\; \mu\sqrt{r}
\Bigr\}$ and by Assumption~\ref{assum:identfinite}(ii):
\begin{equation*}
\frac{1}{\sigma_i^2} + \frac{6}{\kappa_{\cL}}\max\Bigl\{
\lambda\sqrt{s},
\; \mu\sqrt{r}
\Bigr\} \ge \widehat{\S}_{i,i}
\quad \text{and} \quad
\frac{1}{\sigma_j^2}
\;-\; 
\frac{6}{\kappa_{\cL}}\,
\max\Bigl\{
\lambda\sqrt{s},
\; \mu\sqrt{r}
\Bigr\} \le \widehat{\S}_{j,j}.
\end{equation*}
Therefore, we can get for all terminal nodes $i$ and non-terminal nodes $j$: $\widehat{\S}_{i,i} < \widehat{\S}_{j,j}$.
Thus, we can still find the terminal node by finding the minimum value of the diagonal entries of $\widehat{\S}$.

Regarding the correct edge recovery, by Assumption~\ref{assum:identfinite}(ii), the support recovery of $\S$ is correct and since $\B_{i,\bullet}=-\frac{\B_{i,\bullet}}{\B_{i,i}}$, the support of $\B$ is also correct.
\end{proof}

\subsection{Proof of Theorem~\ref{thm:B}}

\begin{proof}
Theorem~\ref{thm:Bdeterministic} is a deterministic theorem, and the error bound and the regularization parameters depend on the randomness in the data.
To prove Theorem~\ref{thm:B}, we first need to show that $\lambda \ge 2\|\widehat{\bSigma}-\bSigma^{*}\|_{\infty}$ and $\mu \ge 2\|\widehat{\bSigma}-\bSigma^{*}\|_{2}$ satisfy with high probability.
The proof is similar to the proof of Theorem~\ref{thm:SL} and thus, we will not repeat the probability proof here.
Under Theorem~\ref{thm:SL}, the number of samples $n$ satisfies $n \ge \max\{4C_{1}^{2}\log p,\; C_{2}^{2}p\}$.

We can get the error bound is $m =\max\Bigl\{
c_{1}\sqrt{\frac{s\log p}{n}},
\; c_{2}\sqrt{\frac{rp}{n}}
\Bigr\}$. In order to guarantee that $\| \widehat{\B}-\B^{\ast}\|_{\infty}
\;\le\; \epsilon$,

\begin{equation*}
c\,m\,(1+\|\B^*\|_\infty) \;\le\; \epsilon
\end{equation*}
Thus, we can get 
\begin{equation*}
n\;\ge\;c^{2}(1+\|\B^*\|_\infty)^{2}/\epsilon^{2}
\max\{
c_1^{2}s\log p,
\; c_2^{2}rp
\}
\end{equation*}
Combined with the requirements previously needed for $n$, we obtain
\begin{equation*}
n\;\ge\; \max\Bigl\{4C_{1}^{2}\log p,\; C_{2}^{2}p, \;
c^{2}(1+\|\B^*\|_\infty)^{2}/\epsilon^{2}
\max\{
c_1^{2}s\log p,
\; c_2^{2}rp
\}
\Bigr\}
\end{equation*}
and we prove our claim.
\end{proof}

\end{document}